\documentclass[letterpaper, 10 pt, conference]{ieeeconf}  % Comment this line out if you need a4paper

\IEEEoverridecommandlockouts                              % This command is only needed if 
\usepackage{graphics,color} % for pdf, bitmapped graphics files
\usepackage{epsfig} % for postscript graphics files
\usepackage{mathptmx} % assumes new font selection scheme installed
\usepackage{times} % assumes new font selection scheme installed
\usepackage{amsmath} % assumes amsmath package installed
\usepackage{amssymb}  % assumes amsmath package installed
\usepackage{bm}
\usepackage{dblfloatfix}
\usepackage{dirtytalk}
\usepackage{textcomp,balance}
\usepackage{amsmath}
\usepackage{mathtools}
\DeclareMathAlphabet{\mathcal}{OMS}{cmsy}{m}{n}
\DeclarePairedDelimiterX{\norm}[1]{\lVert}{\rVert}{#1}
\newtheorem{theorem}{Theorem}
\newtheorem{corollary}{Corollary}[theorem]

\newcommand\given[1][]{\:#1\vert\:}

\title{\LARGE \bf
Predicting optimal value functions by interpolating reward functions in scalarized multi-objective reinforcement learning
}

\author{Arpan Kusari$^{1}$ and Jonathan P. How$^{2}$% <-this % stops a space
\thanks{*This work was supported by Ford Motor Company}% <-this % stops a space
\thanks{$^{1}$Arpan Kusari is with Research and Advanced Engineering, Ford Motor Company, Dearborn, MI 48124, USA
        {\tt\small akusari@ford.com}}%
\thanks{$^{2}$Jonathan P. How is with the MIT Department of Aeronautics and Astronautics, Cambridge, MA 02139, USA
        {\tt\small jhow@mit.edu}}%
}

\begin{document}

\newcommand{\XX}[1]{{\color{red} \bf XX #1 XX}}
\newcommand{\comment}[1]{}

\maketitle
\thispagestyle{empty}
\pagestyle{empty}

%%%%%%%%%%%%%%%%%%%%%%%%%%%%%%%%%%%%%%%%%%%%%%%%%%%%%%%%%%%%%%%%%%%%%%%%%%%%%%%%
\begin{abstract}
A common approach for defining a reward function for multi-objective reinforcement learning (MORL) problems is the weighted sum of the multiple objectives. The weights are then treated as design parameters dependent on the expertise (and preference) of the person performing the learning, with the typical result that a new solution is required for any change in these settings. This paper investigates the relationship between the reward function and the optimal value function for MORL; specifically addressing the question of how to approximate the optimal value function well beyond the set of weights for which the optimization problem was actually solved, thereby avoiding the need to recompute for any particular choice. We prove that the value function transforms smoothly given a transformation of weights of the reward function (and thus a smooth interpolation in the policy space). A Gaussian process is used to obtain a smooth interpolation over the reward function weights of the optimal value function for three well-known examples: Gridworld, Objectworld and Pendulum. The results show that the interpolation can provide robust values for sample states and actions in both discrete and continuous domain problems. Significant advantages arise from utilizing this interpolation technique in the domain of autonomous vehicles: easy, instant adaptation of user preferences while driving and true randomization of obstacle vehicle behavior preferences during training. 
\end{abstract}

%%%%%%%%%%%%%%%%%%%%%%%%%%%%%%%%%%%%%%%%%%%%%%%%%%%%%%%%%%%%%%%%%%%%%%%%%%%%%%%%
\section{INTRODUCTION}
\label{sec:intro}
% MORL and weighted sum
Reinforcement learning (RL) is a machine learning technique that provides the basis for decision-making, where a reward provided by the environment leads the agent to behave in a manner so as to maximize the cumulative sum of rewards. The reward function of RL problems often requires optimization of multiple, often conflicting objectives \cite{roijers2013survey}. For example, in the domain of autonomous vehicles, driving preferences have to be balanced between time to goal, comfort and safety \cite{prokop2001modeling}, which are correlated and its unclear how they influence each other. These conflicting objectives do not yield a single optimal solution, but rather a set of trade-off solutions which balance the objectives \cite{van2014multi}. The easiest way to solve the multi-objective problem is to use a linear scalarization function \cite{hwang2012multiple} that transforms the given problem into a standard single-objective using a weighted sum of the parameters. 

% Reward shaping in literature
% In RL literature, reward shaping, which is the act of providing modifications to the reward function, has been utilized to decrease training time by utilizing prior knowledge about the system. [\textbf{\emph{Ng, Harada and Russell, 1999}}] investigated conditions under which modifications to the reward function would keep the optimal policy invariant. They showed that in addition to the positive linear transformation, a reward for transitions between states can be applied which is the difference in value of an arbitrary potential function applied to those states which maintains the optimal policy. While these results provide an important idea in reward shaping, the authors do not deal with the change in value function as a result of modification of reward function. 

% Problems of using weighted sum
Sutton's reward hypothesis states \say{that all of what we mean by goals and purposes can be well thought of as maximization of the expected value of the cumulative sum of a received scalar signal (reward)}. Thus, the inference being that any given multi-objective problem can always be transformed into a single objective reward function. The most obvious problem in this case is that the weights used during training are a design parameter and dependent on the preference of the person designing the RL problem. Thus, the trained RL has a set optimal policy (and optimal value function) which is dependent on the weights provided. Having a fixed set of weights can be detrimental to the possibility of adaptation to different user experiences whereby for every instance of change of weights, the process of training (which is tedious and time intensive) needs to be repeated. 

% Fundamental question
A question which arises is: Given a small sparse group of optimal value functions under variable reward functions given by different weights, is it possible to interpolate through the entire space of the reward functions to provide exact estimates of optimal value functions at all possible states and actions? 

% [\textbf{\emph{Suay, Brys, Taylor and Chernova, 2017}}] provided 

% Other references
To the best of our understanding, prior research works focusing on value function interpolation have been used to show convergence of RL algorithms for countable and uncountable spaces. Ref.~\cite{davies1997multidimensional} proposed multilinear interpolation techniques on coarse grid to solve various RL paradigms. Ref.~\cite{szepesvari2001convergent} provided convergence of RL algorithms combined with value function interpolation while providing convergence of Q-learning for uncountable spaces. Although it is fairly obvious that changing the reward function would effect the value function directly, we have not found any research work which investigates the relationship and predicts it for weights not previously seen during training. 

% MORL references
The majority of MORL approaches consist of single-policy algorithms in order to learn Pareto optimal solutions \cite{mannion2017policy}. Ref. \cite{barrett2008learning} provides a modification of RL to learn all the optimal policies for all linear preference assignments by incorporating the convex hull of the value function. Ref. \cite{wang2012multi} uses Monte-Carlo Tree Search (MCTS) along with multi-objective indicator by the way of a hypervolume indicator to define action-selection criterion. Ref. \cite{van2014multi}, which uses multi-objective optimization techniques within a RL framework, creates a multi-policy algorithm that learns a set of Pareto dominating policies in a single run of the algorithm which they call Pareto Q-learning. While our proposed approach is useful for MORL problems, we do not aim to create a different MORL approach in this paper. Rather our research formulation is different than the existing MORL approaches in that we seek to derive value functions at unseen reward weights (in the training phase) from the neighboring interpolations. 

% Providing our solution
The aim of this research is to interpolate through the space of the value functions as a result of changing the weights of the reward function using a Gaussian Process (GP). The change in weights may be non-uniform, which makes the process highly nonlinear. Thus, it becomes a supervised learning problem where with the increase in the number of objectives, the weight space increases and data points becomes extremely sparse. Finding accurate value function values across problem space would be extremely beneficial for machine learning in general and autonomous vehicles in particular. GPs provide flexible function approximators, capable of learning intricate structure through their covariance kernels \cite{williams1996gaussian}. Utilizing the predictive power of GPs to interpolate through the high-dimensional input space should yield accurate value functions at all points of the large state space.   

% Paper organization
This paper is organized as follows: Section \ref{sec:background} gives a background of RL and GP, Section \ref{sec:method} provides the claim along with the mathematical reasoning, Section \ref{sec:results} gives the results of the methodology on various standard RL examples, and Section \ref{sec:conclusions} gives the discussions and conclusions. 

\allowdisplaybreaks

% Background
\section{Background}
\label{sec:background}
\subsection{Reinforcement learning}
\label{subsec:rl}
In the RL task, at time t, the agent observes a state, $s_t \in$ S, which represents the environmental model of the system. It takes an action, $a_t \in$ A. The agent receives an immediate scalar reward $r_t$ and moves to a new state $s_{t+1}$. The environment's dynamics are characterized by state transition probabilities $p(s_{t+1}|s_t,a_t)$. This can be formally stated as a Markov Decision Process (MDP) where the next state can be completely defined by the previous state and action (Markov property) and receive a scalar reward for executing the action \cite{bellman1957markovian}.

The goal of the agent is to maximize the cumulative reward (discounted sum of rewards) or value function:
\begin{equation}
\label{eq:value_func}
V_t = \sum_{k=0}^{\infty} \gamma^k R_{t+k}
\end{equation} 
where $0 \leq \gamma \leq 1$ is the discount factor and $r_t$ is the reward at time-step $t$. 
In terms of a policy $\pi : S \rightarrow A$, the value function can be given by Bellman equation as:
\begin{align}
\label{eq:bellman}
V_{\pi}(s_t) &= \mathop{\mathbb{E}}_{\pi} \Bigg[  \sum_{k=0}^{\infty} \gamma^k R_{t+k} | S = s_t \Bigg] \\
 &= \mathop{\mathbb{E}}_{\pi} \Bigg[ R_t + \sum_{k=1}^{\infty} \gamma^k R_{t+k} | S = s_t \Bigg] \\
 &\begin{aligned}
 &= \sum_{a} \pi (a_t|s_t) \sum_{s_{t+1}} p(s_{t+1}|s_t,a_t)\\
 &\qquad \times \Bigg[ R(s_t,a_t,s_{t+1}) + \gamma \mathop{\mathbb{E}}_{\pi}  \Bigg[\sum_{k=0}^{\infty} \gamma^k R_{t+k} | S=s_{t+1} \Bigg] \Bigg] 
 \end{aligned} \\
 &\begin{aligned}
 &= \sum_{a} \pi (a_t|s_t) \sum_{s_{t+1}} p(s_{t+1}|s_t,a_t)\\
 &\qquad \times \Bigg[ R(s_t,a_t,s_{t+1}) + \gamma V_{\pi}(s_t+1) \Bigg]. 
 \end{aligned}
\end{align}
Using Bellman's optimality equation, we can  define a policy $\pi$ which is greater than or equal to any other policy $\pi'$, if value function $V_{\pi}(s_t) \leq V_{\pi'}(s_t)$ for all $s_t \in$ S. This policy is known as an optimal policy ($\pi^*$) and its value function is known as optimal value function ($V^*$).

For continuous state space problems, such as arising in control of nonlinear dynamical systems, a common approach to solve the problem is using value function approach \cite{sutton2000policy}. Value-function approach estimates a value function for each action and chooses the ``greedy" policy (policy having highest value function) at each time-step. Thus, the value function is updated until it converges to the optimal value function.  

\subsection{Gaussian process regression}
\label{subsec:gpr}
A stochastic process is a collection of random variables of functions, $\{f(x):x \in \mathcal{X}\}$, where the variables are collected from a set $\mathcal{X}$.  A GP is a special form of stochastic process, where any finite subset of the random variables has a multivariate Gaussian distribution \cite{rasmussen2003gaussian}. In particular, a collection of random variables $\{f(x):x \in \mathcal{X}\}$ is said to be drawn from a GP with mean function $m(\cdot)$ and covariance function $k(\cdot, \cdot)$, if for any finite set of elements $\{x_1, \cdots, x_n\} \in \mathcal{X}$, the associated finite set of random variables $\{f(x_1), \cdots , f(x_n)\}$
have distribution,
\begin{equation}
\label{eq:gp1}
    \begin{bmatrix}
    f(x_1) \\
    \vdots \\
    f(x_n) \\
    \end{bmatrix}
    \sim 
    \mathcal{N}
    \Bigg(
    \begin{bmatrix}
    m(x_1) \\
    \vdots \\
    m(x_n) \\
    \end{bmatrix},
    \begin{bmatrix}
    k(x_1, x_1) \!\!& \! \cdots \!& k(x_1, x_n) \\
    \vdots \!\!& \!\ddots \!& \vdots \\
    k(x_n, x_1) \!\!& \!\cdots\! & k(x_n, x_n) \\
    \end{bmatrix}
    \Bigg)
\end{equation}
and the resulting GP is then denoted as 
\begin{equation*}
    f(\cdot) \sim \mathcal{GP}(m(\cdot), k(\cdot, \cdot)).
\end{equation*}
While any real-valued function is suitable for mean function $m(\cdot)$, the kernel function $k(\cdot,\cdot)$ needs to guarantee positive-semidefiniteness. 

Let $P = \{(x(i), y(i))\}_{i=1}^n$ be a training set of i.i.d. examples from some unknown distribution. In the Gaussian process regression model,
\begin{equation}
\label{eq:gp2}
    y(i) = f(x(i)) + \epsilon(i), i = \{1, \cdots, n\}
\end{equation}
where the $\epsilon(i)$ are i.i.d. ``noise'' variables with independent $\mathcal{N} (0, \sigma^2)$ distributions. We assume a zero-mean Gaussian process prior, $f(\cdot) \sim \mathcal{GP}(0, k(\cdot, \cdot))$ with a covariance function $ k(\cdot, \cdot)$.  The marginal distribution over any set of input points belonging to $\mathcal{X}$ must have a joint multivariate Gaussian distribution. Therefore, for testing points $Q = \{x^*(i), y^*(i)$, the marginal distribution is given as
\begin{equation}
\label{eq:gp3}
    \begin{bmatrix}
    \vec{f} \\
    \vec{f^*} \\
    \end{bmatrix}
    \given[\Bigg]
    X, X^* \sim \mathcal{N} \left( 0,
    \begin{bmatrix}
    K(X,X) && K(X,X^*) \\
    K(X^*,X) && K(X^*,X^*)\\
    \end{bmatrix} \right)
\end{equation}
where $X$ is the matrix formulation of the training input vector, $X^*$ is the matrix formulation of the test input vector and $\vec{f^*}$ is the compactly written vector formulation of $f(x^*)$. The outputs can therefore be written as:
\begin{multline}
\label{eq:gp4}
    \begin{bmatrix}
    \vec{y} \\
    \vec{y^*} \\
    \end{bmatrix}
    \given[\Bigg]
    X, X^*
    = \begin{bmatrix}
    \vec{f} \\
    \vec{f^*} \\
    \end{bmatrix} + 
    \begin{bmatrix}
    \vec{\epsilon} \\
    \vec{\epsilon^*} \\
    \end{bmatrix} \\
    \sim \mathcal{N} \bigg(0,
    \begin{bmatrix}
    K(X,X)+\sigma^2 I && K(X,X^*) \\
    K(X^*,X) && K(X^*,X^*)+\sigma^2 I\\
    \end{bmatrix})
\end{multline}
where $\epsilon^*(i)$ are i.i.d. ``noise'' variables with independent $\mathcal{N} (0, \sigma^2)$ distributions. We derive the test outputs from Equation \ref{eq:gp4} as:
\begin{equation}
    \vec{y^*} \given[\big] \vec{y}, X, X^* \sim \mathcal{N}(\mu^*, \Sigma^*)
\end{equation}
where 
\begin{equation*}
    \mu^* = K(X^*,X)(K(X,X)+\sigma^2 I)^{-1} \vec{y}
\end{equation*} and 
\begin{equation*}
    \Sigma^* = K(X^*,X^*)+\sigma^2 I - K(X^*,X)(K(X,X)+\sigma^2 I)^{-1}K(X,X^*)
\end{equation*}

% Interpolation proof
\section{Methodology}
\label{sec:method}
\subsection{Value function interpolation}
\label{subsec:value_inter}
In this section, we focus on providing mathematical justifications for the interpolation of value function based on the weights of the objectives of reward function.

For initial analysis, we wish to prove that given a simple, linear transformation of weights, the value function can be interpolated in an accurate manner. Intuitively, we are trying to derive the intermediate optimal value function giving the optimal policy for some MDP, where the reward is the weighted combination of various different objectives. 
\begin{theorem}
\label{theorem:grad-v}
For a reward function $R$ composed of $n$ different objectives, each associated with weight $w_i$, with the full set given by $\mathbf{w} = \{w_1, w_2, ..., w_n\}$, such that for a given state $s_t \in S$ and a given action $a_t \in A$, the reward function is
\begin{equation}
\label{eq:reward}
R(s_t,a_t) = w_1 r_1(s_t,a_t) + w_2 r_2(s_t,a_t) + ... + w_n r_n(s_t,a_t).
\end{equation}
where $r_1, r_2, ..., r_n$ are normalized reward functions at a given state $s_t$ and action $a_t$, respectively, the gradient of the state-value function with  respect to the weights exists, if all the rewards at the current state and action are finite.
\end{theorem}

\begin{proof}
The optimal value at a state is given by the state-value function 
\begin{equation}
\label{eq:state-value-function}
V^*(s_t) = \max_\pi \mathop{\mathbb{E}}\left[\sum_{t=0} \gamma^{t} R(s_t)\right],
\end{equation}
where $R(s_t) = \max_\pi R(s_t, a_t)$.
Given a particular set of weights, we substitute  (\ref{eq:reward}) into (\ref{eq:state-value-function}) to obtain 
\begin{multline}
\label{eq:reward-weight1-state-value}
V^*(s_t|\mathbf{w}) = \max_\pi \mathop{\mathbb{E}}[ \sum_{t=0}^\infty \sum_{a_t \in A} \gamma^t (w_1 r_1(s_t,a_t) + \cdots \\+w_n r_n(s_t,a_t))].
\end{multline}
However, note that 
for a different set of weights $\mathbf{w}' = {w_1', \ldots, w_n'}$, the optimal state-value function is 
\begin{multline}
\label{eq:reward-weight2-state-value}
V^*(s_t|\mathbf{w}') = \max_\pi \mathop{\mathbb{E}}[ \sum_{t=0}^\infty \sum_{a_t \in A} \gamma^t (w_1' r_1(s_t,a_t) +   \cdots \\+w_n' r_n(s_t,a_t))].
\end{multline}

\newpage
Subtracting (\ref{eq:reward-weight1-state-value}) from (\ref{eq:reward-weight2-state-value}) yields
\begin{multline}
\label{eq:deltav-1}
\Delta V^*(s_t| \mathbf{w}, \mathbf{w}') = \\
\max_\pi \mathop{\mathbb{E}}[ \sum_{t=0}^\infty \sum_{a_t \in A} \gamma^t (w_1' r_1(s_t,a_t) + \cdots + w_n' r_n(s_t,a_t))] - \\
\max_\pi \mathop{\mathbb{E}}[ \sum_{t=0}^\infty \sum_{a_t \in A} \gamma^t (w_1 r_1(s_t,a_t) +  \cdots + w_n  r_n(s_t,a_t))]
\end{multline}
Using the property $\max (b) - \max (a) \leq \max (b-a)$ yields
\begin{multline}
\label{eq:deltav-2}
\Delta V^*(s_{t}| \mathbf{w}, \mathbf{w}') \leq \max_\pi \mathop{\mathbb{E}}[\sum_{t=0}^\infty \sum_{a_t \in A} \gamma^{t}
((w_1'-w_1) r_1(s_t,a_t) + ...\\ + (w_n'-w_n) r_n(s_t,a_t))]
\end{multline}
Equation (\ref{eq:deltav-2}) can be written in matrix form as
\begin{equation}
\label{eq:deltav-3}
\Delta V^*(s_{t}| \mathbf{w}, \mathbf{w}') \leq \max_\pi \mathop{\mathbb{E}}\left[\sum_{t=0}^\infty \sum_{a_t \in A} \gamma^{t} \cdot 
\Delta \mathbf{w}\cdot \bm{r}\right]
\end{equation}
where $\Delta \mathbf{w} = \mathbf{w}' -\mathbf{w} $
%\begin{equation*}
%\Delta \mathbf{w} = \left[ \begin{array}{ccc} (w_1'-w_1) & \cdots & (w_n'-w_n) \end{array}\right]^T
%\end{equation*}
and 
\begin{equation*}
\bm{r} = 
\begin{bmatrix}
r_1(s_t, a_t)\\ 
\vdots \\
r_n(s_t, a_t) 
\end{bmatrix}^T
\end{equation*}
Since, $\bm{\Delta w}$ is constant for all states and actions, (\ref{eq:deltav-3}) can be rearranged as 
\begin{equation}
\frac{\Delta V^*(s_{t}| \mathbf{w}, \mathbf{w}')}{\Delta w_i}  \leq \max_\pi \mathop{\mathbb{E}}[\sum_{t=0}^\infty \sum_{a_t \in A} \gamma^{t} \cdot
\bm{1}\cdot \bm{r}]
\end{equation}
which gives the approximate gradient of the value function with respect to the $i^{th}$ weight. If all the rewards at the current state and action are finite, then the gradient will exist for that given state of the MDP.
Thus, the linear interpolation of weights in reward function leads to smooth interpolation of state-value function. 
\end{proof}

\begin{corollary}
Under linear transformation of weights in reward function, the gradient of the action-value function with respect to the weights exists, if all the rewards at the current state and action are finite. 
\end{corollary}
\begin{proof}
For a optimal state-value function $V^*(s)$ that gives the best value at that particular state, the optimal action-value function (optimal value of a state and action combination) is
\begin{equation}
    \label{eq:q-value}
    Q^*(s_t,a_t) = R(s_t,a_t) + \gamma \sum_{t=0}^\infty p(s_{t+1}|s_t,a_t) V^*(s_{t+1}).
\end{equation}
Given two different set of weights, the difference in q-value functions can be written as:
\begin{multline}
    \label{deltaq-1}
    \Delta Q^*(s_t,a_t) = \bm{\Delta w}\cdot \bm{r} + \gamma \sum_{t=0}^\infty p(s_{t+1}|s_t,a_t) \Delta V^*(s_{t+1})
\end{multline}
Replacing $\Delta V^*$ from (\ref{eq:deltav-3}) gives
\begin{multline}
    \label{deltaq-2}
    \Delta Q^*(s_t,a_t) \leq \bm{\Delta w}\cdot \bm{r} + \gamma \sum_{t=0}^\infty p(s_{t+1}|s_t,a_t) \\ \max_\pi \mathop{\mathbb{E}}[\sum_{t=0}^\infty \sum_{a_t \in A} \gamma^{t} \bm{\Delta w}\cdot \bm{r}].
\end{multline}
Therefore the gradient of $Q^*$ with respect to the $i^{th}$ weight is given as:
\begin{multline}
    \label{deltaq-3}
    \frac{\Delta Q^(s_t, a_t)}{\Delta w_i} \leq \bm{1} \cdot \bm{r} + \gamma \sum_{t=0}^\infty p(s_{t+1}|s_t,a_t) \\ \max_\pi \mathop{\mathbb{E}}[\sum_{t=0}^\infty \sum_{a_t \in A} \gamma^{t} \cdot \bm{1} \cdot \bm{r}]
\end{multline}
\end{proof}
%

%\noindent \textbf{Remark:} A similar proof could be provided for the nonlinear interpolation which would also smoothly transform the given value function. 

The shaped reward function is a specific case of the MORL reward function, whereby, the reward function is augmented using an indicator function in which a positive reward is given if the next state is closer to the goal and can be represented as
\begin{equation}
\label{eq:shaping_reward_function}
R'(s_t,a_t,s_{t+1}) = R(s_t,a_t,s_{t+1}) + r \cdot \mathbf{I}(d(s_{t+1},G) < d(s_t,G))
\end{equation}
where $G$ is the goal state. Assuming that the goal state is constant across the different weights of the reward function, the added shaped reward remains constant for the given state across weights. Thus, the reward shaping does not pose any problems for interpolating reward functions. 

\section{Results}
\label{sec:results}
We use three different example tasks with various degrees of complexity to test the validity of our approach. These example tasks have multiple objectives which need to be optimized simultaneously using the RL framework. We change the weights of these objectives and intend to predict the resulting value function using GP regression. 
 
\subsection{Gridworld}
The gridworld \cite{sutton2018reinforcement} is a discrete $N\times N$ grid with four actions per state (corresponding to steps in each direction) and each action has a 10\% chance of being randomly changed to a different action. If the agent hits a wall, then it stays in the previous state. The goal states correspond to a large terminal reward and there is a negative living reward for each of the other states, which incentivizes the agent to reach the goal as fast as possible. There is a walled state in the $(2,2)$ position. The default terminal rewards are $p=+1$ and $n=-1$ and the default living reward is $l=-0.02$. 

The GP regression from Scikit in Python \cite{scikit-learn} is used to determine the interpolated value function, where the input vector $X$ corresponds to a state vector augmented with the discrete action and the weights of the reward function, and the scalar target $y$ corresponds to the value function.  The Matern kernel is utilized for training the GP with default parameters in all the cases. We used other kernels, but we did not find sufficient difference between the kernel choices. 

We vary the living reward and terminal rewards for the different experiments. Two kinds of metrics are reported: the mean squared error between the actual value function and predicted value function over all states and all actions and the  median value of the standard deviation at the query points. Both interpolated and extrapolated query points are reported, which are presented as representative samples. 
\begin{figure*}[t]
    \centering
\includegraphics[scale=0.25]{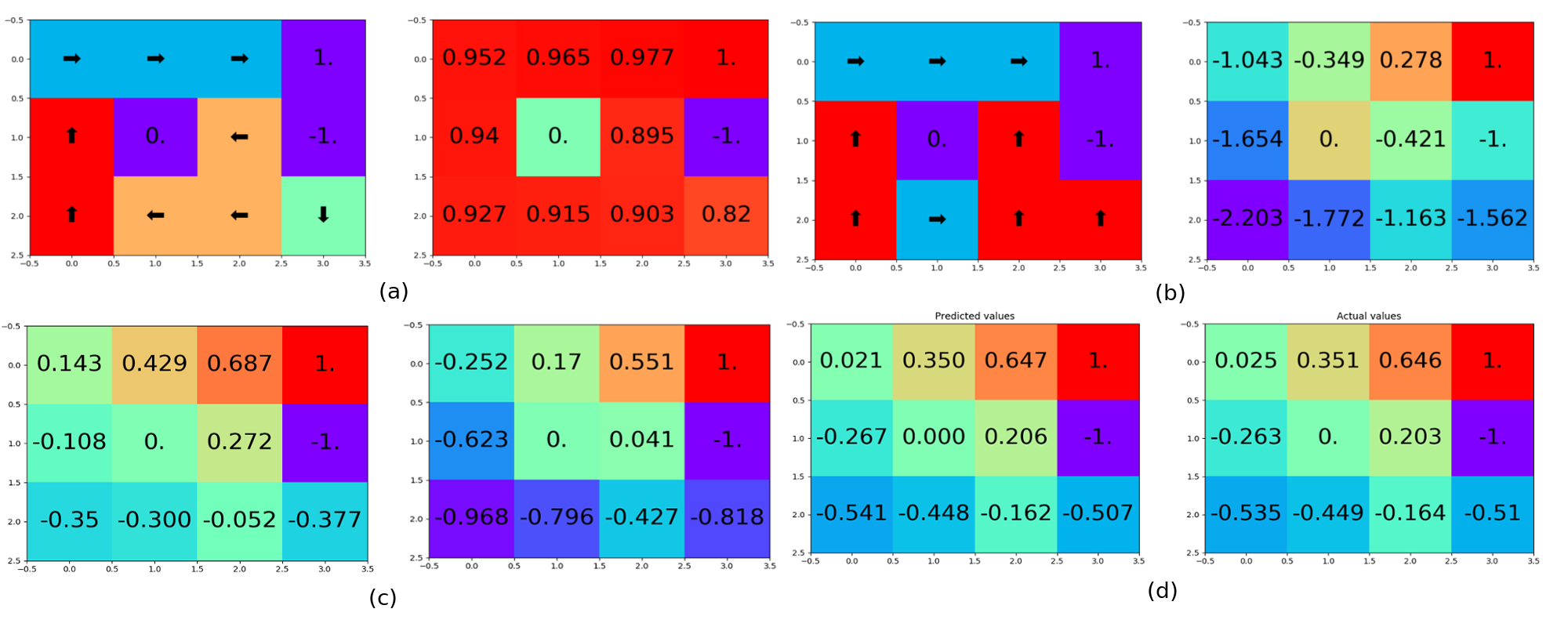}
\vspace*{-.2in}
         \caption{(a) Optimal policy and optimal value function for living reward ($0$) and (b) optimal policy and value function for living reward $-0.5$. (c) For the interpolation of living reward ($-0.23$), we show the optimal value functions for two neighboring points with living reward ($-0.2$) and living reward ($-0.3$). (d) Predicted and actual optimal function values for living reward ($-0.23$). }
        \label{fig:reward_value_living}
\end{figure*}

%\begin{figure}[h]
%    \centering
%        \includegraphics[scale=0.2]{figures/GridWorld_reward_0_23.png}
%        \caption{Reward and State value-function for all the states}
%        \label{fig:living_reward}
%\end{figure}

\subsubsection{Changing the living reward}
We vary the living reward ($l$) of all states (except the terminal states) to vary the optimal policy (and by virtue the optimal state value function) in a way that the variability is nonlinear. The training is performed by varying the living reward from $l=0$ to $l=-0.5$ by steps of $-0.1$. In the Table \ref{tab:living_reward}, the results are presented for four different interpolated evaluation living rewards as well as an extrapolated evaluation living reward. The interpolation results are shown to be accurate to the fourth decimal place while the extrapolation is within a feasible error bound. To understand the effect of variability of living reward on the optimal state-value function and the subsequent optimal policy, Figure~\ref{fig:reward_value_living}(a) and \ref{fig:reward_value_living}(b) shows the optimal state-value function and policy for extreme living rewards $l=0$ and $l=-0.5$, respectively. It is clear from the optimal policies that a change in living reward alters the solution to the gridworld problem sufficiently and there is a need to capture the variability in the living reward. Figure~\ref{fig:reward_value_living}(c) gives the optimal state-value functions at the neighboring points ($l=-0.2$ and $l=-0.3$) of the example living reward $l=-0.23$, which shows substantial variability in the optimal state-value functions. Figure~\ref{fig:reward_value_living}(d) gives the predicted and actual optimal state-value function values at the chosen living reward ($l=-0.23$), which shows that the results are accurate to the third decimal place for all states; thus proving the accuracy of the interpolation for the entire state space. 

\begin{table}[t]
\caption{Predicting value functions for living rewards}
\label{tab:living_reward}
\begin{center}
\begin{tabular}{c c c}
\hline
Living reward (l) & Mean squared error & Median sigma\\
\hline
\rule{0pt}{10pt} -0.16 & 1.019e-04 & 1.732e-03\\
\rule{0pt}{10pt} -0.23 & 1.529e-05 & 1.496e-02\\
\rule{0pt}{10pt} -0.37 & 3.401e-05 & 1.764e-02\\
\rule{0pt}{10pt} -0.45 & 9.273e-04 & 5.111e-02\\
\rule{0pt}{10pt} -0.60 & 4.999e-02 & 2.382e-01\\
\hline
\end{tabular}
\end{center}
\end{table}

\subsubsection{Changing the negative terminal reward}
\begin{table}[b]         
\caption{Predicting value functions for negative terminal rewards}
\label{tab:negative_reward}
\begin{center}
\begin{tabular}{c c c}
\hline
Negative reward (n) & Mean squared error & Median sigma\\
\hline
\rule{0pt}{10pt} -1.3 & 4.098e-03 & 4.246e-03\\
\rule{0pt}{10pt} -2.2 & 2.678e-07 & 8.535e-04\\
\rule{0pt}{10pt} -3.6 & 3.099e-10 & 1.282e-04\\
\rule{0pt}{10pt} -4.7 & 8.290e-08 & 2.582e-04\\
\rule{0pt}{10pt} -6.0 & 7.025e-06 & 2.304e-03 \\ 
\hline
\end{tabular}
\end{center}
\end{table}
The negative terminal reward ($n$) is varied from $n=-1$ to $n=-5$ with steps of $-0.5$. The evaluations are given in the Table \ref{tab:negative_reward}. Again, both interpolation (first four rows) and extrapolation (last row) evaluation cases were considered.
Note that with an increase in magnitude of the negative terminal reward, the value function in the other states is not influenced (due to the max operator) and thus the results only reflect the difference in the negative terminal state. 

\subsubsection{Changing the positive terminal reward}
Table \ref{tab:positive_reward} shows the result when the positive terminal reward ($p$) is changed from $p=1$ to $p=5$ with steps of $0.5$ and evaluated at the same random points (positive in this case) as in Table~\ref{tab:negative_reward}. The results clearly show that, in both interpolation and extrapolation, the GP is able to recover the value functions. 

\begin{table}[t]
\caption{Predicting value functions for positive terminal rewards}
\label{tab:positive_reward}
\begin{center}
\begin{tabular}{c c c}
\hline
Positive reward (p) & Mean squared error & Median sigma\\
\hline
\rule{0pt}{10pt} 1.3 & 3.876e-05 & 1.193e-03\\
\rule{0pt}{10pt} 2.2 & 6.419e-08 & 4.502e-04\\
\rule{0pt}{10pt} 3.6 & 2.372e-10 & 9.059e-04\\
\rule{0pt}{10pt} 4.7 & 6.789e-08 & 7.603e-04\\
\rule{0pt}{10pt} 6.0 & 1.596e-05 & 1.053e-03\\
\hline
\end{tabular}
\end{center}
\end{table}

% \begin{figure}[t]
%     \centering
%        \includegraphics[scale=0.2]{figures/scale_factor_2_2.png}
%        \caption{Reward and State value-function for all the states}
%        \label{fig:negative_scale_factor}
%\end{figure}
%
%\begin{figure}[h]
%    \centering
%        \includegraphics[scale=0.2]{figures/scale_reward_12_2.png}
%        \caption{Reward and State value-function for all the states}
%        \label{fig:positive_scale_factor}
%\end{figure}

% \begin{figure*}[b]
%     \centering
%         \includegraphics[width=2.0\columnwidth]{figures/avg_delta_q_std_delta_q.png}
%         \caption{(a) Average percentage difference of Q-values (actual - predicted) and (b) Standard deviation in percentage of difference in Q-values for a grid of position weights and angle weights}
%         \label{fig:cartpole_avg_std}
% \end{figure*}

\subsection{Objectworld}
Objectworld \cite{levine2011nonlinear} is an extension of gridworld that features random objects placed in the grid (Figure~\ref{fig:objectworld}(a)). The objects are assigned a random outer and inner color (out of $C$ colors) with the state vector being composed of the Euclidean distance to the nearest object with a specific inner or outer color. The true reward is positive in states that are both within $3$ cells of outer color $1$ and $2$ cells of outer color $2$, negative within $3$ cells of outer color $1$, and zero otherwise. Inner colors and all other outer colors are distractors. In the given example, we use two colors, blue and red. Fifteen different objects are placed randomly within the $10\times 10$ grid with randomly chosen inner and outer color. The positive reward is varied from $0.5$ to $1$ with $0.6$, $0.7$ and $0.8$ points being predicted. 

The formulation for GP regression is similar to the ones used in Gridworld. Figure~\ref{fig:objectworld}(b) shows the actual value function while Figure~\ref{fig:objectworld}(c) provides the predicted value function. Table \ref{tab:reward_objectworld} provides the statistics for the given prediction. The interpolation is not accurate as in gridworld due to the nonlinearity of the reward with respect to the states, but the GP can still recover values close to the actual values, especially in the positive reward region. 

\begin{table}[t]
\caption{Predicting value functions for rewards in Objectworld}
\label{tab:reward_objectworld}
\begin{center}
\begin{tabular}{c c c}
\hline
Reward & Mean squared error & Median sigma\\
\hline
\rule{0pt}{10pt} 0.6 &  7.594e-02 & 7.673e-03\\
\rule{0pt}{10pt} 0.7 &  4.571e-02 & 9.608e-03\\
\rule{0pt}{10pt} 0.8 &  2.415e-02 & 1.526e-03\\
\hline
\end{tabular}
\end{center}
\end{table}

\begin{figure*}[h]
    \centering
        \includegraphics[scale=0.21]{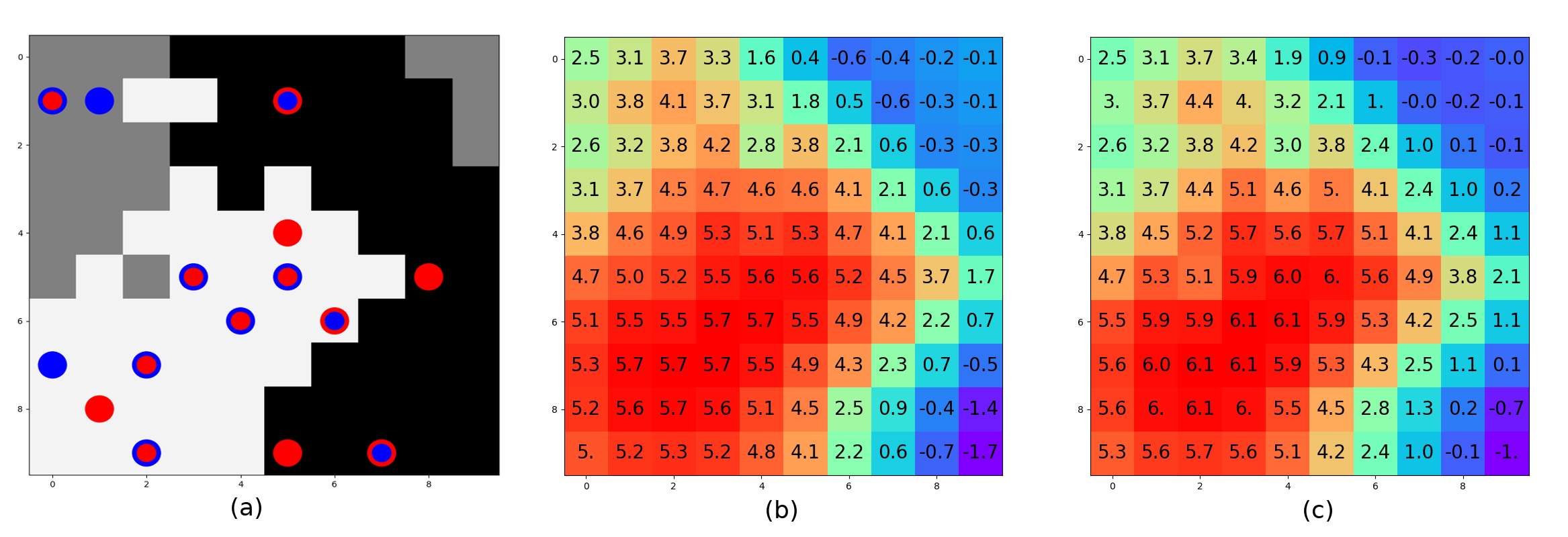}\vspace*{-.15in}
        \caption{(a) Objectworld with 15 randomly placed objects in blue and red inner and outer colors chosen randomly; white represents positive reward, black negative reward and grey zero reward (b) Actual value function for positive reward ($0.8$) (c) Predicted value function}
        \label{fig:objectworld}
\end{figure*}

\begin{figure}[h!]
    \centering
        \includegraphics[trim=60 40 80 10,clip,width=\columnwidth]{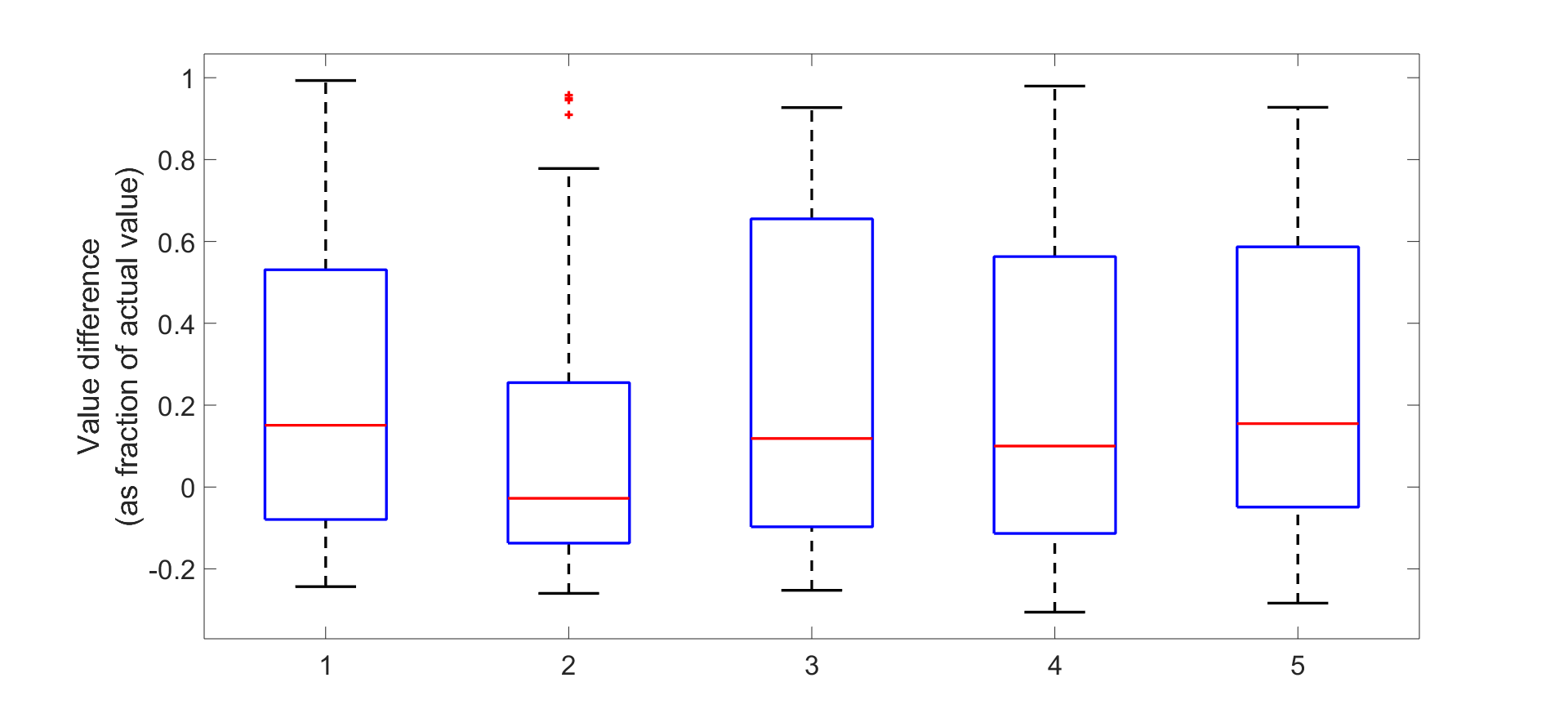}
        \caption{Boxplots for 5 example episodes showing the difference between the predicted and actual values derived for the weight $w3 = 0.001$}
        \label{fig:boxplot}
\end{figure}

\subsection{Pendulum}
% Description of Pendulum
The pendulum environment \cite{brockman2016openai} is an well-known problem in the control literature in which a pendulum starts from a random orientation and the goal is to keep it upright while applying the minimum amount of force. The state vector is composed of the cosine (and sine) of the angle of the pendulum, and the derivative of the angle. The action is the joint effort as $5$ discrete actions linearly spaced within the $[-2, 2]$ range. The reward is 
\begin{equation}
    \label{eq:pendulum_reward}
    R = -(w_1 \cdot \norm[\big]{\theta}^2 + w_2 \cdot \norm[\big]{\dot{\theta}}^2 + w_3 \cdot \norm[\big]{a}^2),
\end{equation}
where $w_1$, $w_2$ and $w_3$ are the reward weights for the angle $\theta$, derivative of angle $\dot{\theta}$ and action $a$ respectively. The optimal reward weights given by OpenAI are $[1, 0.1, 0.001]$ respectively. An episode is limited to 1000 timesteps. 

% Solving with DQN
A deep Q-network (DQN) was proposed in \cite{mnih2015human} that combines deep neural networks with RL to solve continuous state discrete action problems. DQN uses a neural network that gives the Q-values for every action and uses a buffer to store old states and actions to sample from to help stabilize training. The pendulum environment is solved using the DQN approach for various $w_3 =  \{0.1, 0.01, 0.001, 0.0001\}$ with the evaluation performed at $w_3=0.001$. Since this is a continuous state problem, we utilize the trained evaluation model to transition to the next state. 

Utilizing a DQN provides no guarantees that the states seen during testing have been visited during training, which can lead to out-of-distribution states. Thus, we have to utilize a robust Student-t likelihood using the GP regression in GPFlow package \cite{GPflow2017}. The boxplots for the difference in values for 5 sample evaluation episodes are provided in Figure~\ref{fig:boxplot}.  Thus, we use these boxplots to show the value difference as a fraction of the actual value at that state and action. The  boxplots show  that  the  GP  is  able  to  recover  a  value  close  to the actual value (with zero being no difference and greater than $1$ meaning that the predicted value is not able to recover the actual value at all) for the majority of the episodes for continuous state domain problems.  

\section{Conclusions}
\label{sec:conclusions}
This paper shows a direct relationship between the weights of the reward function and the optimal value function for scalarized MORL. This helped us in interpolating through a space of optimal value functions generated using the sparse set of reward functions to estimate the value functions at  sample states. The specific example problems were chosen to understand the value function hypersurface as a function of the reward function. Using GP to interpolate between value functions help us to benefit from prior work in GP regression. Utilizing this relationship would be very beneficial in high-dimensional problems where the instant adaptation of optimal value functions (and thus optimal policies) would save time and cost required for retraining. 

The scalarization approach of MORL is restrictive in that it cannot work with objectives where Pareto fronts are non-convex or have discontinuities \cite{das1997closer}. MORL is an area of active research that uses algorithms leveraged from the multi-objective optimization literature. However, our paper deals with problems which have a defined convex Pareto front and provides a very simple technique in determining optimal value functions at different weights. 

Future work will focus on developing transfer learning of specific behaviors in multi-agent environments with different reward functions based on different weights.

\balance
\bibliographystyle{IEEEtran}
\bibliography{iros}

%%%%%%%%%%%%%%%%%%%%%%%%%%%%%%%%%%%%%%%%%%%%%
%%%%%%%%%%%%%%%%%%%%%%%%%%%%%%%%%%%%%%%%%%%%%%
%%%%%%%%%%%%%%%%%%%%%%%%%%%%%%%%%%%%%%%%%%%%%

\comment{\subsection{Cartpole}
% Description of cartpole
The cartpole environment \cite{brockman2016openai}, also known as the inverted pendulum problem, is a well-known problem in RL whereby a pole is attached to a cart. The pole is unstable and the only way to keep it upright is by moving the cart horizontally left or right by one unit. The episode ends when the pole is more than 15$^\circ$ from vertical, or the cart moves more than $2.4$ units from the center. The state vector is composed of the position of the cartpole, the angle of the cartpole, and their derivatives. The default reward is the amount of time the pole is upright in a given episode (time limit of 500 timesteps). Due to the nature of the continuous state space, it is impossible to visit every state and action pair during training or to know when the training is complete. OpenAI judges the training to be complete when the agent gets an average of 475 (out of a possible 500) or above for 100 consecutive episodes. 

% Solving with DQN
Deep Q-network (DQN) has been proposed by \cite{mnih2015human} which combines deep neural networks with RL to solve continuous state discrete action problems. DQN uses a neural network that gives the Q-values for every action and uses a buffer to store old states and actions to sample from which helps to stabilize training. The cartpole environment can be easily solved using the DQN approach.  

% Our modifications
We modify the reward in two ways: we provide a preferential reward of $+1$ if the position is within $1$ unit from the center and a reward of $+1$ if the angle is less than 5$^\circ$ from vertical. For both of these cases, we have a pair of weights associated with them giving a measure of how important they are to the user. We train the cartpole using DQN on four different weight pair points, $(0,0)$, $(0,1)$, $(1,0)$ and $(1,1)$. The optimal policies for each of the weight points are different for 50\% of the total sampled states. %The difference in the trained policies for the terminal points for 200 points shown in Figure~\ref{fig:policies}. 
Since, we augment the original reward vector, we define the training as complete when the average reward exceeds the following limit for 100 consecutive episodes:
\begin{equation}
\begin{split}
\text{max reward} &= (\text{time limit} - 25) \\
           & ~\quad +\text{weight position} \times \text{time limit} \\
           & ~\quad +\text{weight angle} \times \text{time limit}
\end{split}
\end{equation}
%
%\begin{figure*}[t]
%    \centering
%        \includegraphics[width=2\columnwidth, height=0.05\textheight]{ieeeconf/policies.png}
%        \caption{Difference in policies between the terminal weight points where black is going left and right is going right}
%        \label{fig:policies}
%\end{figure*}
%
We then sample the set of 200 states and all actions and the subsequent reward in the X-vector and the given q-value in the y-vector. We predict the \XX{Q?} q-values for weights starting from $0$ to $1$ with steps of $0.2$ for position and angle \XX{only 1 range given not two, so confused} respectively. Figure~\ref{fig:cartpole_avg_std} gives the average difference and standard deviation of difference between the predicted \XX{} q-values and trained \XX{} q-values scaled by the maximum q-value. We find the maximum average difference to be around 4\% and the maximum standard deviation to be 2\%. The optimal policies resulting from the interpolation agree within a 5\% tolerance from the actual policies. }

\comment{subsection{Autonomous vehicle driving}
We devise a simple simulation environment for autonomous driving with the ego-vehicle initialized at a random distance on a random lane and other obstacle cars initialized within a certain distance of the ego-vehicle, with a random initial speed and a random lane. The simulation environment is shown in Figure~\ref{fig:car_sim}. The obstacle cars exhibit Adaptive Cruise Control (ACC) driving behaviors, keeping a constant speed if outside the safety distance limit or the safe distance if within the limit. The state is given by the longitudinal and lateral affordance indicators \cite{chen2015deepdriving} with respect to the ego vehicle. The ego-vehicle can take five actions: three levels of speed (1, 2 and 3 units/step) and lane change (left and right). For the lane change, the car moves one lane in one time-step laterally. The training episode ends if the car is outside the bounds laterally or if there is a collision between the ego-vehicle and other obstacle cars.

\begin{figure}[t]
    \centering
        \includegraphics[width=1.0\columnwidth,height=0.1\textheight]{ieeeconf/car_sim_environment.png}
        \caption{Simple car simulation environment with the ego-car in red and other cars in blue. The maximum sensing distance is given as a red rectangle.The states are given by the lateral distances from the edge of leftmost lane and rightmost lane and the longitudinal distances and relative velocities to nearest cars in each lane.}
        \label{fig:car_sim}
\end{figure}

A self-driving vehicle has to negotiate between two main factors to drive optimally: comfort and safety. The reward can then be composed of these factors whereby minimizing comfort leads to maximizing lateral acceleration and minimizing safety leads to maximizing longitudinal acceleration, resulting in aggressive behavior. In case of the episode ending early, a large penalty is imposed. 

As in the cartpole example, the four terminal weight pair points (0,0), (0,1), (1,0) and (1,1) are used for training. Three different weight pair points are chosen in a linear fashion: (0.25,0.25), (0.5,0.5),(0.75,0.75). The average error and the median sigma are reported for the three weight points in Table \ref{tab:auto}. 
%Two additional metrics used in evaluation in addition are the average headway distance and the average number of lane changes in a 10 step period. 

\begin{table}[h]
\caption{Interpolating value functions for weight points for Autonomous driving}
\label{tab:auto}
\begin{center}
\begin{tabular}{c c c}
\hline
Weight points & Mean squared error & Median sigma\\
\hline
(0.25,0.25) & XX & XX\\
(0.5,0.5) & XX & XX\\
(0.75,0.75) & XX & XX\\
\hline
\end{tabular}
\end{center}
\end{table}}

\end{document}